\RequirePackage{fix-cm}
\documentclass[smallextended]{svjour3}       
\smartqed  
\usepackage{graphicx}
\usepackage{url}
\usepackage{algorithm}
\usepackage{algpseudocode}
\usepackage{listings}
\usepackage{booktabs}
\usepackage{amsmath}
\usepackage{amssymb}
\usepackage{fancybox}
\usepackage[utf8]{inputenc}
\usepackage{xspace}
\usepackage{tikz}
 \usepackage{multicol}
\usepackage{caption}
\usepackage{subcaption}
\usepackage{verbdef}
\usepackage{cite}
\usetikzlibrary{arrows,fit,positioning}

\DeclareCaptionFormat{subfig}{\figurename~#1#2#3}
\DeclareCaptionSubType*{figure}
\newcommand{\var}[1]{{\ttfamily#1}}
\DeclareMathOperator*{\argmin}{arg\,min}

\begin{document}

\title{MaLeS: A Framework for Automatic Tuning of Automated Theorem Provers
}


\author{Daniel K\"{u}hlwein, Josef~Urban
}


\institute{Daniel K\"{u}hlwein and Josef Urban \at
              Intelligent Systems, Institute for Computing and Information Sciences,\\
	      Radboud University  Nijmegen. \\              
              \email{daniel.kuehlwein@gmail.com, josef.urban@gmail.com}           
}
%

\date{Received: date / Accepted: date}

\maketitle

\begin{abstract}
MaLeS is an automatic tuning framework for automated theorem
provers (ATPs). It provides solutions for both the strategy finding as well as the strategy scheduling problem.
This paper describes the tool and the methods used in it, and evaluates its performance on three automated theorem
provers: E, LEO-II and Satallax. 
An evaluation on a subset of the TPTP library problems shows that on average a MaLeS-tuned prover 
solves $8.67\%$ more problems than the prover with its default settings.
  
\end{abstract}

\section{Introduction}
\label{sec:Introduction}
Automated theorem proving is a search problem.
Many different approaches exist, and most of them have parameters that can be tuned.
Examples of such parameterizations are clause weighting and selection schemes, 
term orderings, and sets of inference and reduction rules used. 
For a given ATP $A$, its implemented parameters form $A$'s \textit{parameter space}. 
A specific choice of parameters is called a search \emph{strategy},\footnote{
Unfortunately, there is no standard terminology for this. 
In Satallax~\cite{Brown2012} parameters are called flags, and a strategy is called a mode.
Option can be used as synonym for parameter. Configurations and configuration space are other alternative names.}
i.e. strategies are elements of the parameter space (Fig. \ref{fig:Problem_Overview}).
The choice of a strategy can often make the difference between finding a proof in a few milliseconds or not at all (within a reasonable time limit).
This naturally leads to the question: Given a new problem, which search strategy should be used?

Considerable attention has already been paid to this problem. 
Gandalf~\cite{Tammet1997} pioneered \textit{strategy scheduling}: 
Instead of running a single strategy for the whole user defined time limit, run several search strategies sequentially for shorter times.
This method is used in most current ATPs, most prominently Vampire~\cite{Riazanov2002}.
In the SETHEO project~\cite{Wolf1998}, a local search algorithm was used to find better strategy schedules.
Fuchs~\cite{Fuchs1998} employed a nearest neighbor algorithm to determine which strategy(s) to run.
Bridge's~\cite{Bridge2010} thesis is about machine learning for search heuristic selection in ATPs with a particular focus on problem features and feature selection.
In the SAT community, Satzilla~\cite{Xu2008} very successfully used machine learning to decide when to run which SAT solver.
ParamILS~\cite{Hutter2009} is a general tuning framework that searches for good parameter settings with a randomized hill climbing algorithm.
BliStr~\cite{Urban2013a} uses ParamILS to develop strategies for E~\cite{Schulz2002} on a large set of interrelated problems.

Despite all this work, most ATPs do not harness the methods available. 
Search strategies are often manually defined by the developer of the ATP and strategy schedules are created by a greedy algorithm or very simple clustering. 
This chapter introduces \emph{MaLeS} (Machine Learning (of) Strategies), a learning-based framework for automatic tuning and configuration of ATPs.
It is based on and supersedes E-MaLeS 1.0~\cite{Kuhlwein2012} and E-MaLeS 1.1~\cite{Kuehlwein2013}.
The goal of MaLeS is to help ATP users to fine-tune an ATP to their problems 
and give developers a simple tool for finding good search strategies and creating strategy schedules. 
MaLeS is implemented in Python and has been tested with the ATPs E, LEO-II~\cite{Benzmueller2008} and Satallax~\cite{Brown2012}. 
The source code is freely available at \url{https://code.google.com/p/males/}.

\subsection{The Strategy Selection Problem}
Figure~\ref{fig:Problem_Overview} gives an informal overview of the \textit{strategy selection problem}.
Given a problem $p \in P$, find a \emph{strategy(s)} $s$ in the \emph{parameter space} $S$ that can quickly solve this problem.
First, we note that parameter spaces can be very big. For example, the ATP E supports over $10^{17}$ different search strategies.
Hence, to simplify the strategy selection problem, strategy selection algorithms usually only consider a small number of \emph{preselected strategies} $\mathfrak{S}$.
Defining $\mathfrak{S}$ is the first challenge. There are different criteria to determine which strategies should be selected.  
The most common ones are to pick strategies that solve a lot of problems, or are very good for a particular kind of problem.

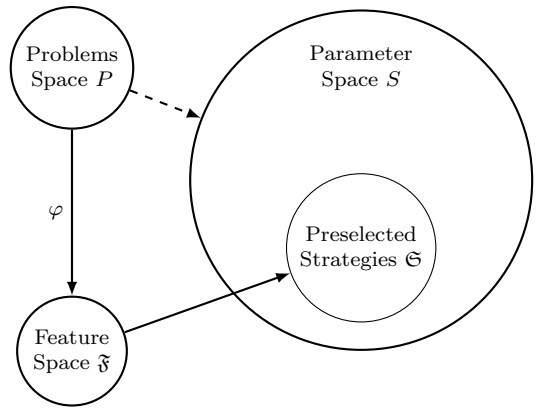
\begin{figure}
\centering
\begin{tikzpicture}
[typetag/.style={rectangle, draw=black!50, font=\scriptsize\ttfamily, anchor=west}]
\tikzstyle{main}=[circle, minimum size = 12mm, thick, draw =black!100, node distance = 22mm,align=center]
\tikzstyle{connect}=[-latex, thick]
\tikzstyle{box}=[rectangle, draw=black!100]
\node [align=center] (s) { Parameter\\Space $S$ };
\node[main, fill=white!100] (p) [left = of s] { Problems\\ Space $P$ };
\node[main, fill=white!100] (f) [below = of p] { Feature\\ Space $\mathfrak{F}$ };
\node [circle,draw =black!100,align=center] (s') [below = of s] {Preselected \\ Strategies $\mathfrak{S}$ };
\node (x) [main, fit={(s) (s')}] {};

\path   (p) edge [connect] node [left] {$\varphi$} (f)
	(p) edge [connect,dashed] (x)
	(f) edge [connect] (s');

\end{tikzpicture}
\caption{Overview of the strategy selection problem for ATPs.}
\label{fig:Problem_Overview}
\end{figure}

As a second step we need a way to characterize problems. 
This is usually done by defining a set of \textit{features} $\mathfrak{F}$.
The features must strike a balance between being fast to compute (via a \emph{feature function} $\varphi$) 
and being expressive enough so that the ATP behaves similarly on problems with similar features.
Once we have defined the features, we still need a way to predict how well each preselected strategy performs on a given set of features.
Finally, one needs to combine the predictions to create a strategy schedule.
Hence, the strategy selection problem consists of three subproblems:
\begin{itemize}
 \item[$\bullet$] Finding a \emph{good} set of preselected strategies $\mathfrak{S}$.
 \item[$\bullet$] Defining features $\mathfrak{F}$ which are easy to compute, but also expressive enough to distinguish different types of problems.
 \item[$\bullet$] Determining a method which given the features of a problem creates a strategy schedule.
\end{itemize}

\subsection{Overview}
\label{sec:Overview}
The rest of the paper is organized as follows:
Section \ref{sec:Finding_Good_Search_Strategies_with_MaLeS} explains how MaLeS defines the preselected strategies $\mathfrak{S}$.
The features and the algorithm that creates the strategy schedule are presented in Section~\ref{sec:Strategy_Scheduling_with_MaLeS}.
MaLeS is evaluated against the default installations of E~1.7, LEO-II~1.6.0 and Satallax~2.7 in Section~\ref{sec:Evaluation}.
The experiments compare the performance of running an ATP in default mode versus running the ATP with strategy scheduling provided by MaLeS.
Future work is considered in Section~\ref{sec:Future_Work}, and the paper concludes with Section~\ref{sec:Conclusion}.
The appendix 
shows how to install the MaLeS-tuned versions of the ATPs mentioned above: E-MaLeS, LEO-MaLeS and Satallax-MaLeS, 
how to tune any of those systems for new problems, and how to use MaLeS with different ATPs. It also includes an overview of the CASC results.

\section{Finding Good Search Strategies with MaLeS}
\label{sec:Finding_Good_Search_Strategies_with_MaLeS}
Choosing a good strategy for a problem requires prior information on how the different strategies behave on different kinds of problems.
Getting this information for all strategies is often infeasible due to constraints on CPU power available and the number possible strategies. 
Hence, one has to decide which strategies one wishes to evaluate.
ATP developers often manually define such a set of strategies based on their intuition and experience.
This option is, however, not available when one lacks in-depth knowledge of the internal workings of the ATP. 
A local search algorithm can help in these cases, and can even be combined with the manual approach by taking the predefined strategies as starting points of the search.

{\small
\begin{algorithm}[th!b]
  \caption{find\_strategies: For each problem search for the strategy that solves it in the least amount of time.}
  \label{alg:findStrategy}
  \begin{algorithmic}[1]
    \Procedure{find\_strategies}{\var{Problems},\var{tol},\var{t\_max},\var{nS},\var{nC}}
    \State initialize Queue $Q$ \label{initQueue}
    \State initialize dictionary \var{bestTime} with \var{t\_max} for all problems
    \State initialize dictionary \var{bestStrategy} as empty
    \While{$Q$ not empty}
    \State $s \gets$ \var{pop}$(Q)$ 
    \For{$p \in \;$\var{Problems}}
    \State \var{oldBestTime} $\gets$ \var{bestTime}$[p]$
    \State \var{proofFound},\var{timeNeeded} $\gets$ \var{run\_strategy}$(s,p,$\var{t\_max}$)$ \label{timeLimit}
    \If{\var{proofFound} \textbf{and} \var{timeNeeded} $<$ \var{bestTime}$[p]$}
      \State \var{bestTime}$[p] \gets$ \var{timeNeeded}
      \State \var{bestStrategy}$[p] \gets s$ 
    \EndIf
    \If{\var{proofFound} \textbf{and} \var{timeNeeded} $<$ \var{bestTime}$[p]+$\var{tol}} \label{local}
      \State \var{randomStrategies} $\gets$ \var{create\_random\_strategies}$(s$,\var{nS},\var{nC}$)$ \label{random}
      \For{$r$ in \var{randomStrategies}}
	\State \var{proofFoundR},\var{timeNeededR} $\gets$ \var{run\_strategy}$(r,p,\text{\var{timeNeeded}})$
	\If{\var{proofFoundR} \textbf{and} \var{timeNeededR}$<$\var{bestTime}$[p]$}
	  \State \var{bestTime}$[p] \gets$ \var{timeNeededR}
	  \State \var{bestStrategy}$[p] \gets r$ 
	\EndIf
      \EndFor
      \If{\var{bestTime}$[p] <$ \var{oldBestTime}}
	\State $Q \gets$ \var{put}$(Q,$\var{bestStrategy}$[p])$ 
      \EndIf
    \EndIf
    \EndFor
    \EndWhile
    \State \Return {\var{bestStrategy}}
    \EndProcedure
  \end{algorithmic}
The initialization of $Q$ in Line~\ref{initQueue} is either done by randomly creating some strategies, or by manually defining which strategies to use. 
Variable \var{tol} defines the tolerance of the algorithm, \var{t\_max} is the maximal time that may be used by the strategy. 
\var{nS} determines the number of strategies generated in the \var{create\_random\_strategies} sub-procedure, 
\var{nC} is an upper limit to how much these new strategies differ from the old one. 
\var{bestStrategy} is a dictionary that for each problems stores the strategy that solved it in the least amount of time.
\end{algorithm}
}

MaLeS employs a basic stochastic local search algorithm labeled \emph{find\_strategies} (Algorithm~\ref{alg:findStrategy}) for ATPs. 
The strategies returned by \emph{find\_strategies} define the \textit{preselected strategies} $\mathfrak{S}$.
The difference to existing parameter selection frameworks like ParamILS and BliStr is that \emph{find\_strategies} searches for each problem for the fastest strategy,
whereas ParamILS tries to find the best strategy for all problems (i.e. find the strategy that 
solves the most problems within some time limit).\footnote{\emph{find\_strategies} is essentially equivalent to running ParamILS on every single problem.} 
BliStr searches for the best strategy for sets of similar problems.

\emph{find\_strategies} takes a list of problems as input. A queue of start strategies is initialized, either with random or predefined strategies.
Each strategy in the queue is then tried on all problems. 
If the strategy solves a problem faster than any of the tried strategies (within some tolerance, see Line~\ref{local}), a local search is performed.
If the search yields faster strategies, the fastest newly found search strategy is appended to the queue.  
In the end, \emph{find\_strategies} returns the strategies that were the fastest strategy on at least one problem.
 
{\small
\begin{algorithm}[th!b]
  \caption{create\_random\_strategies: Returns slight variations of the input strategy.}\label{create_random_strategies}
  \begin{algorithmic}[1]
    \Procedure{create\_random\_strategies}{\var{Strategy},\var{nS},\var{nC}}
    \State \var{newStrategies} is an empty list
    \For{\var{\_i} in range(\var{nS})}\label{numberOfStrategies}    
    \State \var{newStrategy} is a copy of \var{Strategy}
    \For{\var{\_j} in range(\var{nC})}\label{numberOfChanges}
    \State \var{newStrategy} $=$ \var{change\_random\_parameter}$($\var{newStrategy}$)$    
    \EndFor
    \State \var{newStrategies}.\var{append}$($\var{newStrategy}$)$
    \EndFor
    \State \Return{\var{newStrategies}}
    \EndProcedure
  \end{algorithmic}
\var{nS} determines the number of new strategies, \var{nC} is the upper limit for the number of changed parameters.
\end{algorithm}
}
 
The local search part is defined in Algorithm~\ref{create_random_strategies} (\emph{create\_random\_strategies}). 
It returns a predefined number of strategies similar to the input strategy.
The new strategies are created by randomly changing the parameters of the input strategy. 
How many parameters are changed is determined in MaLeS' configuration file.\footnote{Parameter \textit{WalkLength} in Table \ref{tab:setup.ini}}
 
\section{Strategy Scheduling with MaLeS}
\label{sec:Strategy_Scheduling_with_MaLeS}
As mentioned previously, most automated theorem provers, independent of the parameters used, solve problems either very fast, or not at all (within a reasonable time limit).
Instead of trying only a single strategy for a long time, it is often beneficial to run several search strategies for a shorter time.
This approach is called \textit{strategy scheduling}.

Many current ATPs use strategy scheduling to define their default configuration.
Some use a single schedule for every problem (e.g. Satallax 2.7). 
Others define classes of similar problems and use different schedules for different classes (e.g. E 1.7, LEO-II 1.6.0).
MaLeS creates an \textit{individual strategy schedule} for each problem, depending on the problem's \emph{features}. 

\subsection{Notation}
\label{sec:Notation}
We shall use the following notation:
\renewcommand{\labelitemi}{$\cdot$}
\begin{itemize}
 \item $p$ is an ATP problem. $P$ denotes a set of problems.
 \item $P_\text{train} \subseteq P$ is a set of training problems that is used to tune the learning algorithm.
 \item $\mathfrak{F}$ is the feature space. We assume that $\mathfrak{F}$ is a subset of $\mathbb{R}^n$ for some $n \in \mathbb{N}$.
 \item $\varphi : P \rightarrow \mathfrak{F}$ is the \textit{feature function}. $\varphi(p)$ is the feature vector of a problem.
 \item $S$ is the parameter space, $\mathfrak{S}$ is the set of preselected strategies.
 \item The time the ATP running strategy $s$ needs to solve a problem $p$ is denoted by $\tau(p,s)$. 
If $s$ is obvious from the context or irrelevant, we also use $\tau(p)$.
 \item For a strategy $s$, $\rho_s : P \rightarrow \mathbb{R}$ is the runtime prediction function.
\end{itemize}

For each strategy $s$ in the preselected strategies $\mathfrak{S}$, 
MaLeS defines a runtime \textit{prediction function} $\rho_s : P \rightarrow \mathbb{R}$.
The prediction function $\rho_s$ uses the features of a problem to predict the time the ATP running strategy $s$ needs to solve the problem.
The strategy schedule for the problem is created from these predictions.

\subsection{Features}
\label{sec:Features}
Features give an abstract description of a problem. 
Optimally, the features should be designed in such a way that the ATP behaves similar on problems with similar features,
i.e. if two problem $p,q$ have similar features $\varphi(p) \sim \varphi(q)$, 
then for each strategy $s$ the runtimes should be similar $\tau(p,s) \sim \tau(q,s)$.
The similarity function (e.g. cosine distance between the feature vectors) and set of features heavily influence the quality of the prediction functions. 
Indeed, feature selection is an entire subfield of machine learning~\cite{Liu1998,Guyon2003}.

Currently, MaLeS supports two different feature spaces:
Schulz's E features are used for first order (FOF) problems. 
The TPTP features designed by Sutcliffe are used for higher order (THF) problems \cite{Sutcliffe2010a}.
Note that the main reason for using these features was that they were easily available.
Evaluating different features sets and/or introducing new features is beyond the scope of this paper. 

\subsubsection{The E Features}
Schulz designed a set of features for clause-normal-form and first order problems.
They are used in the strategy selection process in his theorem prover E~\cite{Kuehlwein2013}.
Table~\ref{EFeatures} shows the features together with a short description.\footnote{The author would like to thank Stephan Schulz 
for the design of the features, the program that extracts them and their precise description in this subsection.}
MaLeS uses the same features for first-order problems.
A clause is called \emph{negative} if it only has negative
literals. It is called \emph{positive} if it only has positive
literals. A ground clause is a clause that contains no variables.  In
this setting, we refer to all negative clauses as ``goals'', and to all
other clauses as ``axioms". Clauses can be \emph{unit} (having only a
single literal), \emph{Horn} (having at most one positive literal), or
\emph{general} (no constraints on the form). All unit
clauses are Horn, and all Horn clauses are general.

\begin{table}[htbp]
\centering
\caption{Problem features used for strategy selection in E and in first-order MaLeS.}
\resizebox{\columnwidth}{!}{%
\begin{tabular}{lp{0.65\linewidth}}
\toprule
Feature & Description \\\midrule
\texttt{axioms}   & Most specific class (unit, Horn, general) describing
all axioms \\
\texttt{goals}    & Most specific class (unit, Horn) describing
all goals \\
\texttt{equality} & Problem has no equational literals, some
equational literals, or only equational literals\\
\texttt{non-ground units} & Number (or fraction) of unit axioms that
are not ground\\
\texttt{ground-goals} & Are all goals ground?\\
\texttt{clauses} &  Number of clauses \\
\texttt{literals} &  Number of literals \\
\texttt{term\_cells} &  Number of all (sub)terms \\
\texttt{unitgoals} &  Number of unit goals (negative clauses)\\
\texttt{unitaxioms} &  Number of positive unit clauses \\
\texttt{horngoals} &  Number of Horn goals (non-unit)\\
\texttt{hornaxioms} &  Number of Horn axioms (non-unit)\\
\texttt{eq\_clauses} & Number of unit equations\\
\texttt{groundunitaxioms} & Number of ground unit axioms \\
\texttt{groundgoals} & Number of ground goals\\
\texttt{groundpositiveaxioms} & Number (or fraction) of positive
axioms that are ground\\
\texttt{positiveaxioms} & Number of all positive axioms\\
\texttt{ng\_unit\_axioms\_part} & Number of non-ground unit axioms\\
\texttt{max\_fun\_arity} & Maximal arity of a function or predicate symbol \\
\texttt{avg\_fun\_arity} & Average arity of symbols in the problem \\
\texttt{sum\_fun\_arity} & Sum of arities of symbols in the problem\\
\texttt{clause\_max\_depth} &  Maximal clause depth\\
\texttt{clause\_avg\_depth} &  Average clause depth\\
 \bottomrule
\end{tabular}
}
\label{EFeatures}
\end{table}

The features are computed by running Schulz's \emph{classify\_problem} program which is distributed with MaLeS.

\subsubsection{The TPTP Features}
The TPTP problem library~\cite{Sutcliffe2009} provides a syntactical description of every problem which can be used as problem features.
Figure~\ref{TPTPfeatures} shows an example.
Before normalization, the feature vector corresponding to the example is $$[145,5,47,31,1106,\hdots,147,0,0,0,0]$$
Sutcliffe's MakeListStats computes these features and is publicly available as part of the TPTP infrastructure.
A modified version which outputs only the numbers without any text is also distributed with MaLeS.

\begin{figure}[htb!]
\centering
\footnotesize{
\begin{verbatim}
% Syntax   : Number of formulae    :  145 (   5 unit;  47 type;  31 defn)
%            Number of atoms       : 1106 (  36 equality; 255 variable)
%            Maximal formula depth :   11 (   7 average)
%            Number of connectives :  760 (   4   ~;   4   |;   8   &; 736   @)
%                                         (   0 <=>;   8  =>;   0  <=;   0 <~>)
%                                         (   0  ~|;   0  ~&;   0  !!;   0  ??)
%            Number of type conns  :  235 ( 235   >;   0   *;   0   +;   0  <<)
%            Number of symbols     :   52 (  47   :)
%            Number of variables   :  147 (   3 sgn;  29   !;   6   ?; 112   ^)
%                                         ( 147   :;   0  !>;   0  ?*)
%                                         (   0  @-;   0  @+) 
\end{verbatim}
}
\caption{The TPTP features of the THF problem AGT029\textasciicircum1.p in TPTP-v5.4.0.}
\label{TPTPfeatures}
\end{figure}

\subsubsection{Normalization}
In the initial form, there can be great differences between the values of different features.
In the THF example (Figure~\ref{TPTPfeatures}), the number of atoms ($1106$) is of a different order of magnitude than e.g. the maximal formula depth ($7$).
Since our machine learning method (like many other) computes the euclidean distance between data points, these differences can render smaller valued features irrelevant.
Hence, normalization is used to scale all features to have values between $0$ and $1$. 
First we compute the features for each $p \in P_\text{train}$.
Then the maximal and minimal value of each feature $f$ is determined.
These values are then used to rescale the feature vectors for each problem $p$ via
$$\varphi(p)_f :=  \frac{\varphi(p)_f-\text{min}_f}{\text{max}_f-\text{min}_f}$$
where $\varphi(p)_f$ is the value of feature $f$ for problem $p$,
$\text{min}_f$ is the minimal and $\text{max}_f$ is the maximal value for $f$ among the problems in $P_\text{train}$.

\subsection{Runtime Prediction Functions}
\label{sec:Runtime_Prediction_Functions}
Predicting the runtime of an ATP is a classic regression problem \cite{Bishop2006}.
For each strategy $s$ in the preselected strategies $\mathfrak{S}$, we are searching for a function $\rho_s : P \rightarrow \mathbb{R}$ such that
for all problems $p \in P$ the predicted values are close to the actual runtimes: $\rho_s(p) \sim \tau(p,s)$.
This section explains the learning method employed by MaLeS as well as the data preparation techniques used.

\subsubsection{Timeouts}
\label{sec:Timeouts}
The prediction functions are learned from the behavior of the preselected strategies on the training problems $P_\text{train}$.
Each preselected strategy is run on all training problems with a timeout $t$.
Often, strategies will not solve all problems within the timeout. This leads to the question how one should treat unsolved problems.
Setting the time value of an unsolved problem-strategy pair $(p,s)$ to the timeout, i.e. $\tau(p,s) = t$, is one possible solution.
Another possibility, which is used in MaLeS, is to only learn on problems that can be solved. 
While ignoring unsolved problems introduces a bias towards shorter runtimes, it also simplifies the computation of the prediction functions and
allows us to \emph{update} the prediction functions at runtime (Section \ref{sec:Creating_Schedules_from_Prediction_Functions}).

\subsubsection{Kernel Methods}
\label{sec:Kernel_Methods}
MaLeS uses \textit{kernels} to learn the runtime prediction function. 
Kernels are a very popular machine learning method that has successfully been applied in many domains \cite{Shawe-Taylor2004}.
A kernel can be seen as a similarity function between feature vectors. 
Kernels allow the usage of nonlinear features while keeping the learning problem itself linear.  
The basic principles will be covered on the next pages.
More information about kernel-based machine learning can be found in \cite{Shawe-Taylor2004}.

\begin{definition}[Gaussian Kernel]
The Gaussian kernel $k$ with parameter $\sigma$ of two problems $p,q \in P$ with feature vectors $\varphi(p),\varphi(q) \in \mathfrak{F} \subseteq \mathbb{R}^n$ for some $n \in \mathbb{N}$ is defined as
\[
k(p,q) := \exp\left(-\frac{ \varphi(p)^T \varphi(p)  -  2 \varphi(p)^T \varphi(q) + \varphi(q)^T \varphi(q)}{\sigma^{2}}\right)
\]
$\varphi(p)^T$ is the transposed vector, and hence $\varphi(p)^T \varphi(q)$ is the dot product between $\varphi(p)$ and $\varphi(q)$ in $\mathbb{R}^n$.
\end{definition}

In order to apply machine learning, we first need some data to learn from.
Let $t \in \mathbb{R}$ be a time limit. 
For each preselected strategy $s \in \mathfrak{S}$, the ATP is run with strategy $s$ and time limit $t$ on each problem in $P_\text{train}$.
Note that the same $t$ is used for all problems. 
For each strategy $s$, $P_\text{train}^s \subseteq P_\text{train}$ is the set of problems that the ATP can solve within the time limit $t$ with strategy $s$.

\begin{definition}[The Prediction Function]
In kernel based machine learning, the prediction function $\rho_s$ has the form 
$$\rho_s(p) = \sum_{q \in P_\text{train}^s} \alpha^s_q k(p,q)$$ 
for some $\alpha^s_q \in \mathbb{R}$. 
The $\alpha^s_q$ are called weights and are the result of the learning. 
To define how exactly this is done, some more notation is needed.
\end{definition}

\begin{definition}[Kernel Matrix, Times Matrix and Weights Matrix]
For every strategy $s \in \mathfrak{S}$, let $m$ be the number of problems in $P_\text{train}^s$ and $(p_i)_{i \in m}$ be an enumeration of the problems in $P_\text{train}^s$. 
The kernel matrix $K^s \in \mathbb{R}^{m \times m}$ is defined as
\[
K^s_{i,j} := k(p_i,p_j)
\]
We define the time matrix $Y^s \in \mathbb{R}^{1 \times m}$ via
\[
Y_{i}^s := \tau(p_i,s)
\]
Finally, we set the weight matrix $A^s \in  \mathbb{R}^{m \times 1}$ as
\[
A_i^s := \alpha^s_{p_i}
\]
If is it obvious which strategy is meant, or the statement is independent of the strategy, we omit the $^s$ in $K^s$,$Y^s$ and $A^s$.
\end{definition}

A simple way to define values for the weights $\alpha^s_{p_i}$ would be to solve $KA = Y$.
Such a solution (if it exists) would likely perform very well on known data but poorly on new data, a behavior called \textit{overfitting}.
As a measure against overfitting, a regularization parameter $\lambda \in \mathbb{R}$ is added and 
least square regression is used to minimize the difference between the predicted times and the actual times \cite{Rifkin2003}. That means we want
\[
A = \argmin_{A \in \mathbb{R}^{m \times 1}} \left( \left(Y - KA\right)^T \left(Y - KA\right) + \lambda A^T K A \right)
\]
The first part of the equation $\left(Y - KA\right)^T \left(Y - KA\right)$ is the square loss between the predicted values and the actual time needed.
$\lambda A^T K A$ is the regularization term. $A^T K A$ is a measure of how complex, in terms of VC dimension \cite{Vapnik1995}, our prediction function is. 
The bigger $\lambda$, the more complex functions are penalized. 
For very high values of $\lambda$, we force $A$ to be almost equal to the $0$ matrix. 
This approach can be seen as a kind of Occam's razor for prediction functions. 
$A$ is the matrix that best fits the training data while staying as simple as possible.

\begin{theorem}[Weight Matrix for a Strategy]
For $\lambda > 0$, the optimal weights for a strategy $s$ are given by 
\[
A = (K+\lambda I)^{-1} Y
\] 
with $I$ being the identity matrix in $\mathbb{R}^{m \times m}$.
\end{theorem}
\begin{proof}
\begin{eqnarray*}
& \frac{\partial}{\partial A}\left( (Y-KA)^\textrm{T} (Y-KA)+\lambda A^\textrm{T}KA \right)\\
= &-2K(Y-KA)
+2\lambda KA\\
= &-2K Y
+(2KK+2\lambda K)A
\end{eqnarray*}
It can be shown that $K$ is a positive-semi definite symmetric matrix and therefore $(K+\lambda I)$ is invertible for $\lambda > 0$.
To find a minimum, we set the derivative to zero and solve with respect to $A$. 
\begin{eqnarray*}
K(K+\lambda I)A &=&KY
\end{eqnarray*} 
and hence
\begin{eqnarray*}
A &=&(K+\lambda I)^{-1} Y
\end{eqnarray*} 
is a solution.
\end{proof}

\subsection{Crossvalidation}
\label{sec:Crossvalidation}
Finally, the values for the regularization constant $\lambda$ and the kernel width $\sigma$ need to be determined. 
This is done via $10$-fold \textit{cross-validation} on the training problems, a standard machine learning method for such tasks \cite{Kohavi1995}.
Cross-validation simulates the effect of not knowing the data and picks the values that perform, in general, best on unknown problems.

First a finite number of possible values for $\lambda$ and $\sigma$ is defined. 
Then, the training set $P_\text{train}^s$ is split in $10$ disjoint, equally sized subsets $P_1 ,\hdots P_{10}$.
For all $1 \leq i \leq 10$, each possible combination of values for $\lambda$ and $\sigma$ is trained on $P_\text{train}^s -  P_i$ and evaluated on $P_i$. 
The evaluation is done by computing the square-loss between the predicted runtimes and the actual runtimes.
The combination with the least average square loss is used.

\subsection{Creating Schedules from Prediction Functions}
\label{sec:Creating_Schedules_from_Prediction_Functions}
MaLeS uses the knowledge of how different strategies perform on a set of training problems
to estimate how these strategies will behave on a new problem. This is done by learning
runtime prediction functions as described above using the data gathered with Algorithm 1. With the runtime prediction functions
we can create individual strategy schedules for new problems, i.e. compute a strategy schedule for
every set of features.

Given a new problem, MaLeS iterates between computing the predicted runtimes for each strategy, running the predicted best strategy and updating the prediction models.
Algorithm~\ref{males} shows the details.

\begin{algorithm}[htb!]
  \caption{males: Tries to solve the input problem within the time limit. Creates and runs a strategy schedule for the problem.}\label{males}
  \begin{algorithmic}[1]
    \Procedure{males}{\var{problem},\var{time}}
\State \var{proofFound},\var{timeUsed} $\leftarrow$ \var{run\_start\_strategies}(\var{problem},\var{time}) \label{tryStratStrategies}
\If{\var{proofFound}}
\State \Return \var{timeUsed}
\EndIf
\While{\var{timeUsed} $<$ \text{time}} 
\State Set \var{times} as an empty list
\For{$s \in \mathfrak{S}$}
\State $t_s \leftarrow \rho_s($\var{problem}$)$
\State \var{times}.\var{append}$([t_s,s])$
\EndFor
\State $([t_{s'},s']) \leftarrow\;$\var{choose\_best\_strategy}$($\var{times}$)$ \label{chooseBest}
\State \var{proofFound},\var{timeNeeded} $\leftarrow$ \var{run\_strategy}$(s',$\var{problem}$,t_{s'})$ \label{runStrategy}
\State \var{timeUsed} $+=$ \var{timeNeeded}
\If{\var{proofFound}}
\State \Return \var{timeUsed}
\EndIf
\For{$s \in \mathfrak{S}$}
\State \var{timeUsed} $+=$ \var{update\_prediction\_function($\rho_s$,$s'$,$t_{s'}$)} \label{updateModels}
\EndFor
\EndWhile
\State \Return timeUsed
    \EndProcedure
  \end{algorithmic}
\end{algorithm}

In line \ref{tryStratStrategies} the algorithm starts by running some predefined start strategies.
The goal of running these start strategies first is to filter out simple problems which allows the learning algorithm to focus on the harder problems. 
The start strategies are picked greedily. First the strategy that solves most problems (within some time limit) is chosen.
Then the strategy that solves most of the problems that were not solved by the first picked strategy (within some time limit) is picked, etc.
The number of start strategies and their runtime are determined via their respective parameters in the \emph{setup.ini} file (Table~\ref{tab:setup.ini}).
Training problems that are solved by the start strategies are deleted from the training set.
For example, let $s_1,\hdots,s_n$ be the starting strategies, all with a runtime of $1$ second. 
Then for all $s \in S'$ we can set 
$$P_\text{train}^s := \{p \in P_\text{train}^s \mid \forall \; 1\leq i\leq n \; \tau(p,s_i) > 1 \}$$
and train $\rho_s$ on the updated $P_\text{train}^s$.

The subprocedure \var{choose\_best\_strategy} in line \ref{chooseBest} picks the strategy 
with the minimum predicted runtime among those that have not been run with a bigger or equal runtime before.\footnote{If there are several strategies with the same
minimal predicted runtime a random one is chosen.}
\var{run\_strategy} runs the ATP with strategy $s'$ and time limit $t_{s'}$ on the \var{problem}.
If the ATP cannot solve the problem within the time limit, this information is used to improve the prediction functions 
in \var{update\_prediction\_function} (Line \ref{updateModels}).
For this, all the training problems that are solved by the picked strategy $s'$ within the predicted runtime $t_{s'}$ are deleted 
from the training set $P_\text{train}$, i.e. for all $s \in S'$
$$P_\text{train}^s := \{p \in P_\text{train}^s \mid \tau(p,s') > t_{s'} \}$$
Afterwards, new prediction functions are learned on the reduced training set. 
This is done by first creating a new kernel and time matrix for the new $P_\text{train}^s$ and then computing new weights as shown in Theorem 1.
Due to the small size of the training dataset, this can be done in real time during a proof.
Note that these updates are local, i.e. do not have any effect on future calls to \textsc{males}.
If \textsc{males} finds a proof, the total time needed is returned to the user.

\section{Evaluation}
\label{sec:Evaluation}
MaLeS is evaluated with three different ATPs: E 1.7\footnote{E 1.7 was the current version of E when the experiments were done. 
Several significant changes were introduced in E 1.8, in particular new strategies and E's own strategy scheduling. 
As a result, E 1.8 performs better than both E 1.7 and E-MaLeS 1.2. We hope to remedy this situation in the next version of MaLeS.}, LEO-II 1.6 and Satallax 2.7. 
For every prover, a set of training and testing problems is defined. 
MaLeS first searches for good strategies on the training problems using Algorithm \ref{alg:findStrategy} with a $10$ second time limit, i.e. $t_\text{max}=10$.
Promising strategies are then run for $300$ seconds on all training problems.
The resulting data is used to learn runtime prediction functions and strategy schedules as explained in the previous section.
After the learning, MaLeS uses Algorithm 3 when trying to solve a new problem. 
The difference between the different MaLeS versions (i.e. E-MaLeS, Satallax-MaLeS and Leo-MaLeS) is the training data used to create the prediction functions
and start strategies, and the ATP that is run in the \var{run\_strategy} part of Algorithm 3.
The MaLeS version of the ATP is compared with the default mode on both the test and the training problems.
The section ends with an overview of previous versions of MaLeS and their CASC performance.

\subsection{E-MaLeS}
\label{sec:E-MaLeS}
E is a popular ATP for first order logic. 
It is open source, easily available and consistently performs very well at the CASC competitions.
Additionally, E is easily tunable with a big parameter space\footnote{The parameter space considered in the experiments contains more than $10^{17}$ different strategies.} 
which suggested that parameter tuning could lead to significant improvements.
All computations were done on a $64$ core AMD Opteron Processor 6276 with $1.4$GHz per CPU and $256$ GB of RAM

\subsubsection{E's Automatic Mode}
E's automatic mode is developed by Stephan Schulz and based on a static partitioning of the set of all problems into disjoint classes. 
It is generated in two steps. First, the set of all training examples (typically the set of
all current TPTP problems) is classified into disjoint classes using
some of the features listed in Table~\ref{EFeatures}. For the numeric features,
threshold values have originally been selected to split the TPTP into
3 or 4 approximately equal subsets on each feature. Over time, these
have been manually adapted using trial and error.

Once the classification is fixed, a Python program assigns to
each class one of the strategies that solves the most examples in this
class. For \emph{large} classes (arbitrarily defined as having more
than 200 problems), it picks the strategy that also is on average the fastest on that
class. For small classes, it picks the globally best strategy among
those that solve the maximum number of problems. A class with zero
solutions by all strategies is assigned the overall best strategy.

\subsubsection{The Training Data}
The problems from the FOF divisions of CASC-22~\cite{Sutcliffe2010}, CASC-J5~\cite{Sutcliffe2011}, CASC-23~\cite{Sutcliffe2012} 
and CASC-J6 and CASC@Turing~\cite{Sutcliffe2013} were used as training problems.
Several problems appeared in more than one CASC. There are also a few problems from earlier CASCs that are not part of the TPTP version used in the experiments, TPTP-v5.4.0.
Deleting duplicates and missing problems leaves $1112$ problems that were used to train E-MaLeS.
The strategy search for the set of preselected strategies took three weeks on a $64$ core server. 
The majority of the time was spent running promising strategies with a $300$ seconds time limit.
Over $2$ million strategies were considered.
Of those, $109$ were selected to be used in E-MaLeS. E-MaLeS runs $10$ start strategies, each with a $1$ second time limit.
E 1.7 (running the automatic mode) and E-MaLeS were evaluated on all training problems with a $300$ second time limit.
The results can be seen in Figure \ref{fig:etrainperformance}.

\begin{figure*}[!htb]
\centering
  \includegraphics[width=0.75\textwidth]{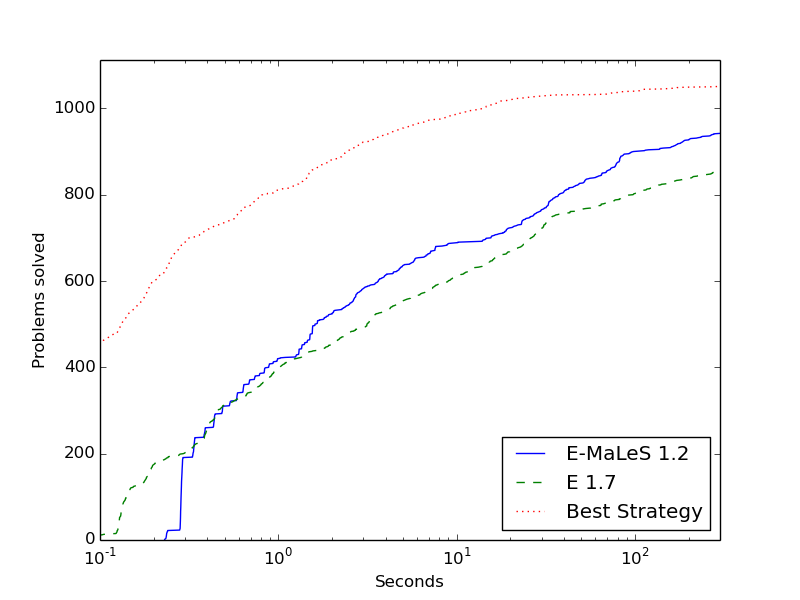}
\caption{Performance graph for E-MaLeS 1.2 on the training problems.}
\label{fig:etrainperformance}       
\end{figure*}

Altogether, $1055$, or $94.9\%$, of the problems can be solved by E 1.7 with the considered strategies.
E 1.7's automatic mode solves $856$ of the problems ($77.0\%$), E-MaLeS solves $10.0\%$ more problems: $942$ ($84.7\%$).  
\emph{Best Strategy} shows the best possible result, i.e. the number of problems solved if for each problem the strategy 
that solves it in the least amount of time was picked.

\subsubsection{The Test Data}
Similar to the way the problems for CASC are chosen, $1000$ random FOF problems of TPTP-v5.4.0 with a difficulty rating \cite{Sutcliffe2001}
between $0.2$ and (including) $1.0$ were chosen for the test dataset.
$165$ of the test problems are also part of the training dataset.

The results are similar to the results on the training problems and can be seen in Figure~\ref{fig:etestperformance}. 
In the first three seconds, E solves more problems than E-MaLeS. Afterwards, E-MaLeS overtakes E.
After $300$ seconds, E-MaLeS solves $573$ of the problems ($57.3\%$) and E 1.7 $511$ ($51.1\%$), an increase of $12.4\%$. 
Figure~\ref{fig:enewtestperformance} shows the results for only the $835$ problems that are not part of the training problems.

\begin{figure*}[!htb]
\centering
  \includegraphics[width=0.75\textwidth]{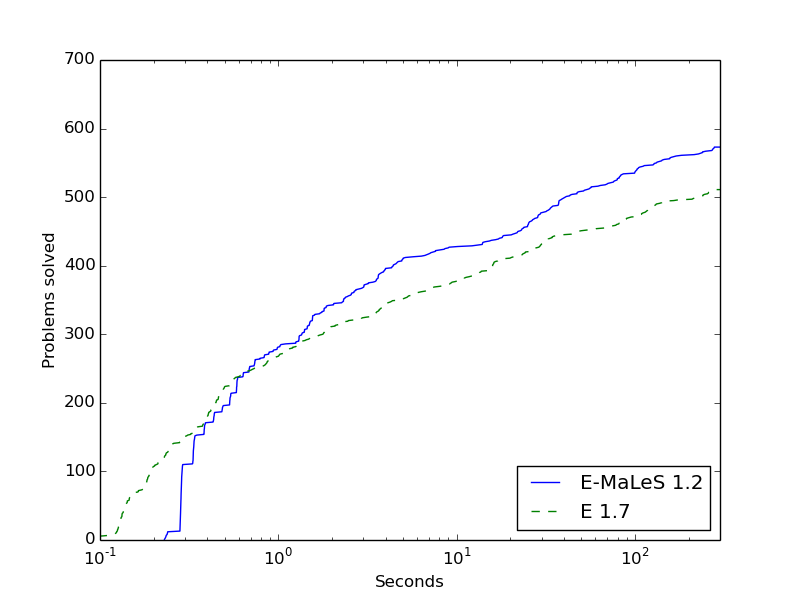}
\caption{Performance graph for E-MaLeS 1.2 on the test problems.}
\label{fig:etestperformance}       
\end{figure*}

\begin{figure*}[!htb]
\centering
  \includegraphics[width=0.75\textwidth]{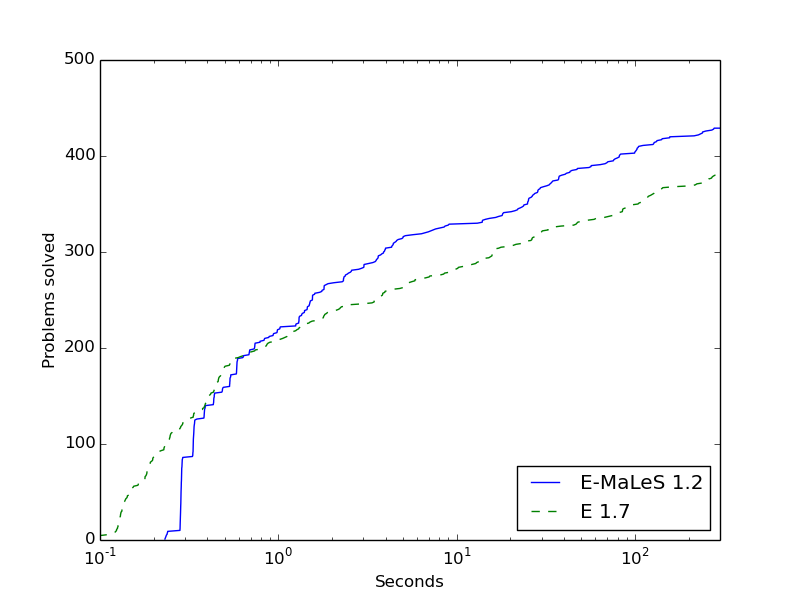}
\caption{Performance graph for E-MaLeS 1.2 on the unseen test problems.}
\label{fig:enewtestperformance}       
\end{figure*}

\subsection{Satallax-MaLeS}
\label{sec:Satallax-MaLeS}
In order to show that MaLeS works for arbitrary ATPs, we picked a very different ATP for the next experiment: Satallax. 
Satallax is a higher order theorem prover that has a reputation of being highly tuned.
The built-in strategy schedule of Satallax solves $95.3\%$ of all solvable problems in the training dataset and, 
with the right parameters, $91.3\%$ ($525$) of the training problems can be solved in less than $1$ second.
The strategy search for the set of preselected strategies was done on a $32$ core Intel Xeon with $2.6$GHz per CPU and $256$ GB of RAM.
The evaluations were done on a $64$ core AMD Opteron Processor 6276 with $1.4$GHz per CPU and $256$ GB of RAM.

\subsubsection{Satallax's Automatic Mode}
Satallax employs a hard-coded strategy schedule that defines a sequence of strategies together with their runtimes.
The same schedule is used for all problems. 
It is defined in the file \textit{satallaxmain.ml} in the \textit{src} directory of the Satallax installation.
Many modes are only run for a very short time ($0.2$ seconds).
This can cause problems if Satallax is run on CPUs that are slower than the one(s) used to create this schedule.

\subsubsection{The Training Data}
\label{satallax:training_data}
The problems from the THF divisions of CASC-J5~\cite{Sutcliffe2011}, CASC-23~\cite{Sutcliffe2012} and CASC-J6~\cite{Sutcliffe2013} were used as training problems.
The THF division of CASC-J5 contained $200$ problems, of CASC-23 $300$ problem, and of CASC-J6 also $200$ problems.
After deleting duplicates and problems that are not available in TPTP-v5.4.0, $573$ problems remain.
The strategy search took approximately $3$ weeks. 
In the end, $111$ strategies were selected to be used in Satallax-MaLeS.
Satallax-MaLeS runs $20$ start strategies, each with a $0.5$ second time limit.

$533$ of the $573$ problems are solvable with the appropriate strategy.
Satallax and Satallax-MaLeS were evaluated on all training problems with a $300$ second time limit.
Satallax solves $508$ of the problems ($88.7\%$). 
Satallax-MaLeS solves $1.6\%$ more problems for a total of $516$ solved problems ($90.1\%$).  

\begin{figure*}[!htb]
\centering
  \includegraphics[width=0.75\textwidth]{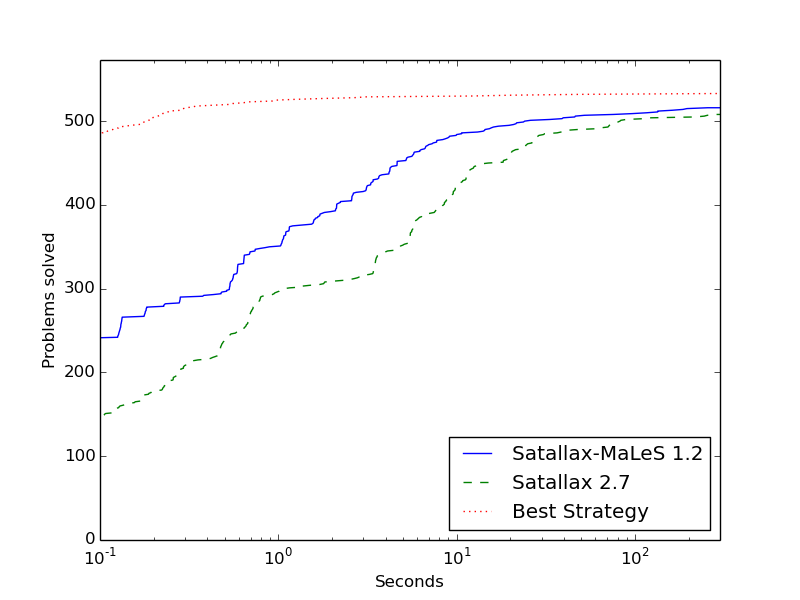}
\caption{Performance graph for Satallax-MaLeS 1.2 on the training problems.}
\label{fig:satallaxtrainperformance}      
\end{figure*}

Figure~\ref{fig:satallaxtrainperformance} shows a log-scaled time plot of the results.
For low time limits, Satallax-MaLeS solves significantly more problems than Satallax.
It seems that Satallax's automatic mode is very suboptimal which might be a result of only focusing on the number of problems solved after $300$ seconds.
\emph{Best Strategy} shows the best possible result, i.e. the number of problems solved if for each problem the strategy 
that solves it in the least amount of time was picked.

\subsubsection{The Test Data}
Similar to the E-MaLeS evaluation, the test dataset consists of $1000$ randomly selected THF problems of TPTP-v5.4.0 with a difficulty rating between $0.2$ and (including) $1.0$.
$301$ of the test problems are also part of the training dataset.
The results are similar to the results on the training problems and can be seen in Figure~\ref{fig:satallaxtestperformance}. 
While the end results are almost the same with Satallax-MaLeS solving $590$ ($59.0\%$ ) and Satallax solving $587$ ($58.7\%$) of the problems, 
Satallax-MaLeS significantly outperforms Satallax for lower time limits.

Figure~\ref{fig:satallaxtestnewperformance} shows the results for only the $699$ problems that are not part of the training problems.
Here, Satallax-MaLeS solves more problems than Satallax in the beginning, but fewer for longer time limits.
After $300$ seconds, Satallax solves $344$ and Satallax-MaLeS $336$ problems.

\begin{figure*}[!htb]
\centering
  \includegraphics[width=0.75\textwidth]{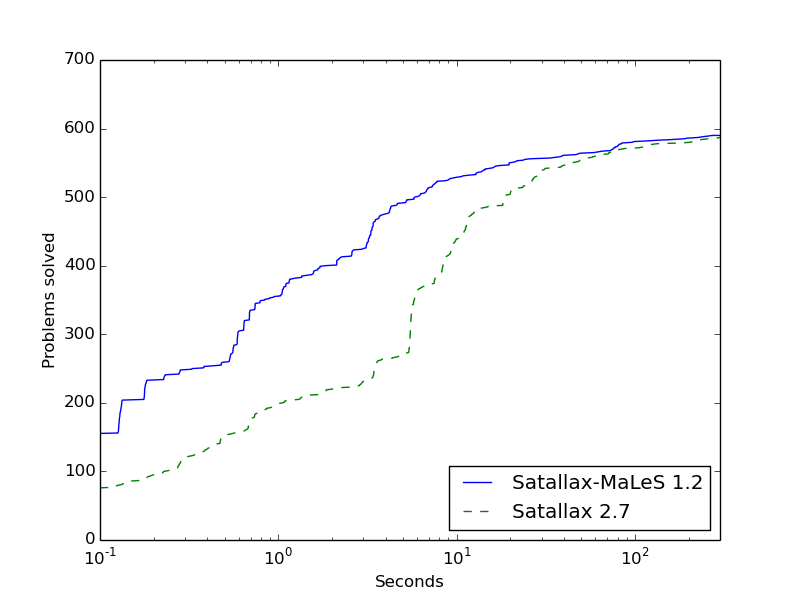}
\caption{Performance graph for Satallax-MaLeS 1.2 on the test problems.}
\label{fig:satallaxtestperformance}      
\end{figure*}

\begin{figure*}[!htb]
\centering
  \includegraphics[width=0.75\textwidth]{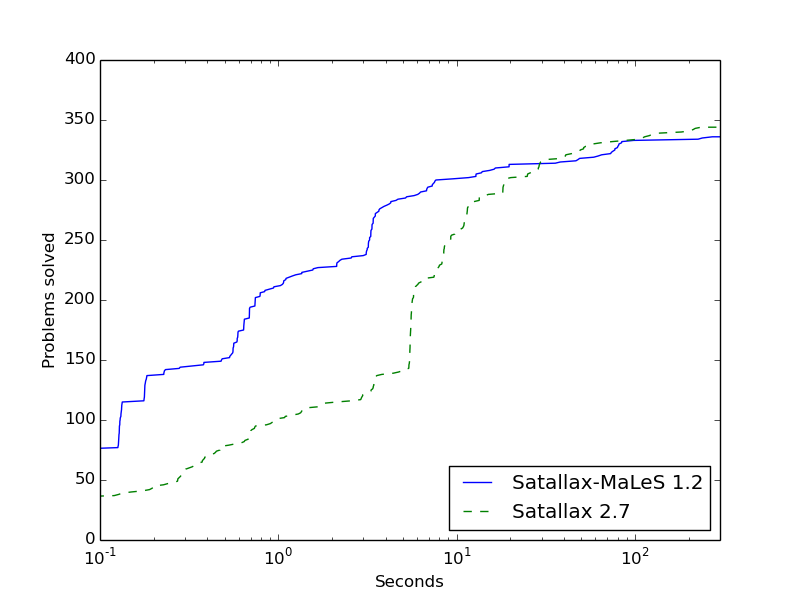}
\caption{Performance graph for Satallax-MaLeS 1.2 on the unseen test problems.}
\label{fig:satallaxtestnewperformance}   
\end{figure*}

\subsection{LEO-MaLeS}
\label{sec:LEO-MaLeS}
LEO-MaLeS is the latest addition to the MaLeS family.
LEO-II is a resolution-based higher-order theorem prover designed for fruitful cooperation with specialist provers 
for natural fragments of higher-order logic.\footnote{Description from the LEO-II website \url{www.leoprover.org}.}
The strategy search for the set of preselected strategies, and all evaluations were done on a $32$ core Intel Xeon with $2.6$GHz per CPU and $256$ GB of RAM.

\subsubsection{LEO-II's Automatic Mode}
LEO-II's automatic mode is a combination of E's and Satallax's automatic modes.
The problem space is split into disjoint subspaces and a different strategy schedule is used for each subspace.
The automatic mode is defined in the file \textit{strategy\_scheduling.ml} in the \textit{src/interfaces} directory of the LEO-II installation.

\subsubsection{The Training and Test Datasets}
The same training and test problems as for the Satallax evaluation were used.
The strategy search took $2$ weeks. 
$89$ strategies were selected.
LEO-II and LEO-MaLeS were run with a $300$ second time limit per problem. 

Of the $573$ training problems $472$ can be solved by LEO-II if the correct strategy is picked.
LEO-MaLeS runs $5$ start strategies, each with a $1$ second time limit. 
Using more start strategies only marginally increases the number of solved problems by the start strategies.
LEO-II's default mode solves $415$ of the training problems ($72.4\%$), and $367$ of the test problems ($36.7\%$).
LEO-MaLeS improves this to $441$ ($77.0\%$) and $417$ ($41.7\%$) solved problems respectively.
Figure~\ref{fig:leotrainperformance} and Figure~\ref{fig:leotestperformance} show the graphs.
Figure~\ref{fig:leotestnewperformance} shows the results for only the $699$ problems that are not part of the training problems.

\begin{figure*}[!htb]
\centering
\includegraphics[width=0.75\textwidth]{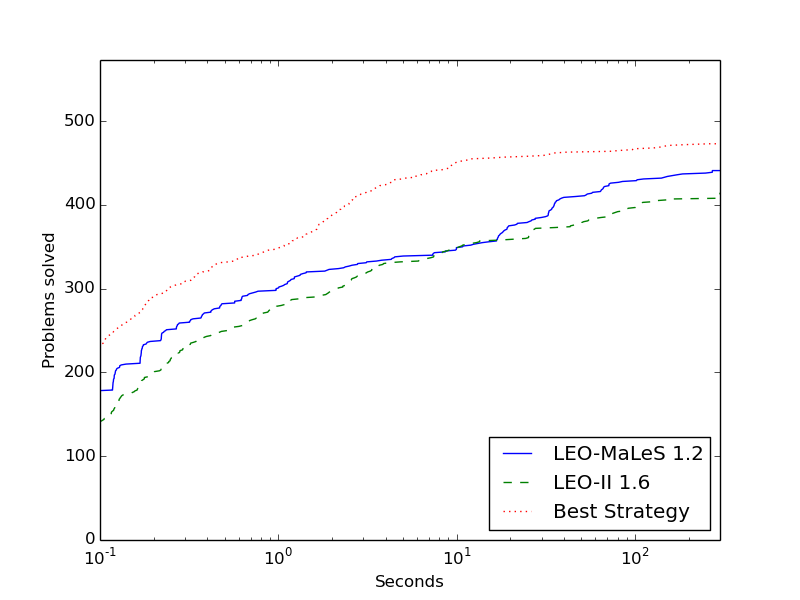}
\caption{Performance graph for LEO-MaLeS 1.2 on the training problems.}
\label{fig:leotrainperformance}    
\end{figure*}

\begin{figure*}[!htb]
\centering
\includegraphics[width=0.75\textwidth]{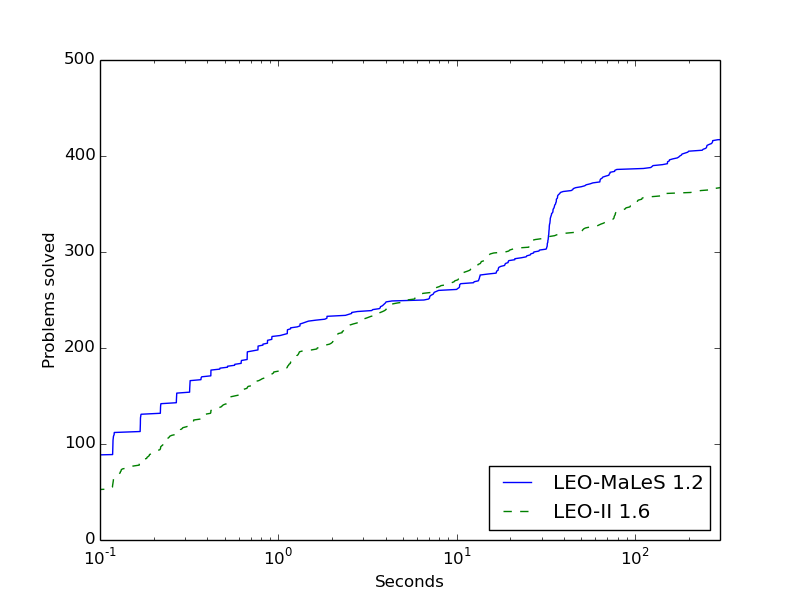}
\caption{Performance graph for LEO-MaLeS 1.2 on the test problems.}
\label{fig:leotestperformance}       
\end{figure*}

\begin{figure*}[!htb]
\centering
\includegraphics[width=0.75\textwidth]{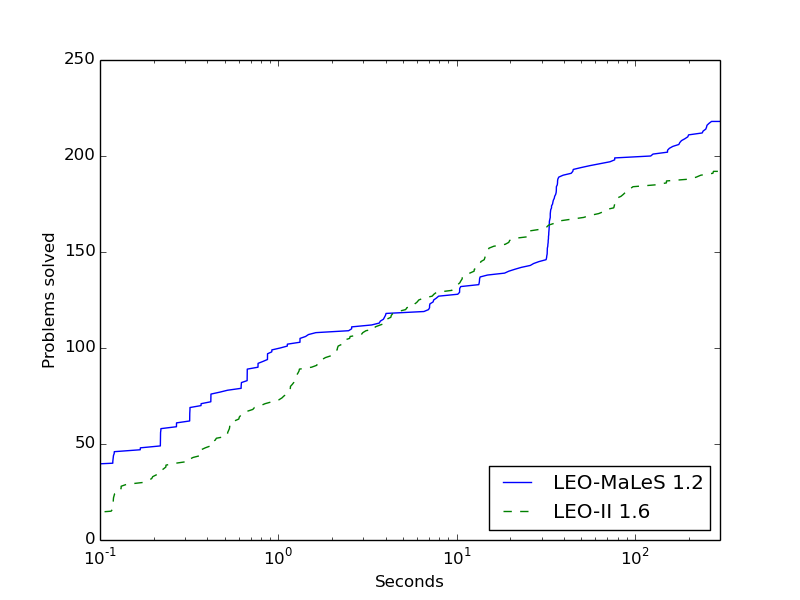}
\caption{Performance graph for Leo-MaLeS 1.2 on the unseen test problems.}
\label{fig:leotestnewperformance}     
\end{figure*}

Between $7$ and $20$ seconds, both provers solve approximately the same number of problems. 
For all other time limits, LEO-MaLeS solves more.
On the test problems, a similar time frame is problematic for LEO-MaLeS. 
LEO-II solves more problems than LEO-MaLeS between $5$ and $30$ seconds.
For other time limits, LEO-MaLeS solves more problems than LEO-II .
This behavior indicates that the initial predictions of LEO-MaLeS are wrong. 
Better features could help remedy this problem.
The sudden jump in the number of solved problems at around $30$ seconds on the test dataset seems peculiar.
Upon inspection, we found that $42$ out of $43$ problems solved in the $30 - 35$ seconds time frame are from the SEU (Set Theory) problem domain.
These problems have very similar features and hence MaLeS creates similar strategy schedules. 
$34$ of the $43$ problems were solved by the same strategy. 

\subsection{Further Remarks}
There are a few things to note that are independent of the underlying prover.

\paragraph{Multi-core Evaluations:}
All the evaluations were done on multi-core machines, a $64$ core AMD Opteron Processor 6276 with $1.4$GHz per CPU and $256$ GB of RAM 
and a $32$ core Intel Xeon with $2.6$GHz per CPU and $256$ GB of RAM. 
All runtimes were measured in wall-clock time.
During the evaluation we noticed irregularities in the runtime of the ATPs.
When running a single instance of an ATP, the time needed to solve a problem often differed from the result we got when running several instances in parallel, 
even when using less than the maximum number of cores.
It turns out that the number of cores used during the evaluation heavily influences the performance.
The more cores, the worse the ATPs performed. 
We were not able to completely determine the cause, but the speed of the hard disk drive, shared cache and process swapping are all possible explanations.
Reducing the hard disk drive load by changing the behavior of MaLeS from loading all models at the very beginning to only when they are needed did lead to 
more (and faster) solved problems.
Eventually, all evaluation experiments (apart from the strategy searches for the sets of preselected strategies) were redone using 
only $20$ out of $64$ / $14$ out of $32$ cores and the results reported here are based on those runs.

\paragraph{How Good are the Predictions?}
Apart from the total number of solved problems, the quality of the predictions is also of interest.
In short, they are not very good. 
The predictions of MaLeS are already heavily biased because the unsolvable problems are ignored (Section \ref{sec:Timeouts}).
Reducing the number of training problems during the update phase makes the predictions even less reliable.
For some strategies, the average difference between the actual and predicted runtimes exceeds $40$ seconds.
Two heuristics were added to help MaLeS to deal with this uncertainty. 
First, the predicted runtime must always exceed the minimal runtime of the training data. 
This prevents unreasonably low (in particular negative) predictions.
Second, if the number of training problems is less than a predefined minimum (set to $5$) then the predicted runtime is the maximum runtime of the training data.
That MaLeS nevertheless gives good results is likely due to the fact that the tested ATPs all utilize either no or very basic strategy scheduling.

\paragraph{The Impact of the Learning Parameters:}
Table~\ref{tab:setup.ini} shows the learning parameters of MaLeS.
\emph{Tolerance}, \emph{StartStrategies} and \emph{StartStrategiesTime} had the greatest impact in our experiments.
Tolerance influences the number of strategies used in MaLeS. A low value means more strategies, a high value less.
For E and LEO, higher values ($1.0 - 15.0$ seconds) gave better results since fewer irrelevant strategies were run.
Satallax performed slightly better with a low tolerance which is probably due to the fact that it can solve almost every problem in less than a second.
The values for StartStrategies and StartStrategiesTime determine how many problems are left for learning.
$10$ StartStrategies with a $1$ second $StartStrategiesTime$ are good default values for the provers tested. 
For LEO-II we found that the number of solved problems barely increased after $5$ seconds, and hence changed to number of StartStrategies to $5$.

\section{Future Work}
\label{sec:Future_Work}
Apart from simplifying the installation and set up, there are several other ways to improve MaLeS. We present the most promising ones.

\paragraph{Automated Parameter Configuration:}
\label{sec:FF:AutomatedParameterConfiguration}
Parameters like \emph{Tolerance}, \emph{StartStrategies} and \emph{StartStrategiesTime} could and should be set automatically. 
We hope to implement this in the next version of MaLeS.

\paragraph{Features:}
\label{sec:FF:Features}
The quality of the runtime prediction function is limited by the quality of the features.
Adding new features and/or integrating feature selection algorithms could increase the prediction capabilities of MaLeS.

\paragraph{Strategy Finding:}
\label{sec:FF:StratFinding}
As an alternative to randomized hill climbing, different search algorithms should be supported.
In particular simulated annealing and genetic algorithms seem promising.
The biggest problem of the current implementation, the time it needs to find good strategies, could be improved
by using a clusterized local search principle similar to the one employed in BliStr~\cite{Urban2013a}.

\paragraph{Strategy Prediction:}
\label{sec:FF:StratPrediction}
The runtime prediction function are the heart of MaLeS. 
Machine learning offers dozens of different regression methods which could be used instead of the kernel methods of MaLeS.
A big drawback of the current approach is that it scales badly due to the need to invert a new matrix after every tried strategy. 
One possible solution for eliminating the need for matrix computations and also the dependency on Numpy and Scipy would be a nearest neighbor algorithm.

\section{Conclusion}
\label{sec:Conclusion}
Finding the best parameter settings and strategy schedules for an ATP is a time consuming task that often requires in-depth knowledge of how the ATP works.
MaLeS is an automatic tuning framework for ATPs that, given the possible parameter settings of an ATP and a set of problems, 
finds good search strategies and creates individual strategy schedules.
MaLeS currently supports E, LEO-II and Satallax and can easily be extended to work with other provers.

Experiments with the ATPs E, LEO-II and Satallax showed that the MaLeS version performs at least comparable to the respective default strategy selection algorithm.
In some cases, the MaLeS optimized version solves considerably more problems than the untuned ATP.

MaLeS aims to simplifies the workflow for both ATP users and developers.
It allows ATP users to fine-tune ATPs to their specific problems and helps ATP developers to focus on actual improvements instead of time-consuming parameter tuning.

\begin{acknowledgements}
The authors were supported by the Nederlandse organisatie voor Wetenschappelijk Onderzoek (NWO) projects 
``MathWiki: A Web-based Collaborative Authoring Environment for Formal Proofs'' and ``Learning2Reason''.

Christoph Benzm\"uller, Chad Brown, Stephan Schulz and Geoff Sutcliffe made this work possible by publicly releasing their programs and providing support whenever problems occurred. 
We would like to thank the anonymous reviewers, Jasmin Christian Blanchette and Michael Nahas for their comments on earlier versions of this paper.
\end{acknowledgements}

\bibliography{ate14}
\bibliographystyle{spmpsci}      

\section{Appendix}
\subsection{Using MaLeS}
\label{sec:Using_MaLeS}
MaLeS aims to be a general ATP tuning framework.
In this section, we show how to setup E-MaLeS, LEO-MaLeS and Satallax-MaLeS, tuning any of those provers on new problems, and how to use MaLeS with a completely new prover.
The first step is to clone the MaLeS git repository via
\begin{center}
 \texttt{git clone} \url{https://code.google.com/p/males/} 
\end{center}
MaLeS requires Python 2.7, Numpy 1.6 or later, and Scipy 0.10  or later \cite{Oliphant2007}. 
Installation instructions for Numpy and Scipy can be found at \url{http://www.scipy.org/install.html}.

\subsubsection{E-MaLeS, LEO-MaLeS and Satallax-MaLeS}
\label{sec:X-MaLeS}
Setting up any of the presented systems can be done in three steps.
\begin{enumerate}
 \item Install the ATP (E, LEO-II or Satallax)
 \item Run the configuration script with the location of the prover as argument. 
      For example \begin{verbatim}EConfig.py --location=../E/PROVER\end{verbatim} 
      for E-MaLeS.
 \item Learn the prediction function via \begin{verbatim}MaLeS/learn.py\end{verbatim} 
\end{enumerate}
After the installation, MaLeS can be used by running
\begin{verbatim}
 MaLeS/males.py -t 30 -p test/PUZ001+1.p
\end{verbatim}
where $-t$ denotes the time limit and $-p$ the problem to be solved.

\subsubsection{Tuning E, LEO-II or Satallax for a New Set of Problems}
\label{sec:Tuning_New_Data}
Tuning an ATP for a particular set of problems involves finding good search strategies and learning prediction models.
The search behavior is defined in the the file \emph{setup.ini} in the main directory.
Using the default search behavior, E, LEO-II and Satallax can be tuned for new problems as follows:

\begin{enumerate}
 \item Install the ATP (E, LEO-II or Satallax)
 \item Run the configuration script with the location of the prover as argument. 
      For example \begin{verbatim}EConfig.py --location=../E/PROVER\end{verbatim} 
      for E-MaLeS.
 \item Store the absolute pathnames of the problems in a new file with one problem per line 
and change the \emph{PROBLEM} parameter in \emph{setup.ini} to the file containing the problem paths.
  \item Find promising strategies by searching with a short time limit (which is the default setup) 
\begin{verbatim}
MaLeS/findStrategies.py 
\end{verbatim}
  \item (Optional) Run all promising strategies for a longer time. For this several parameters need to be changed.
\begin{enumerate}
 \item Copy the value of \emph{ResultsDir} to \emph{TmpResultsDir}.
 \item Copy the value of \emph{ResultsPickle} to \emph{TmpResultsPickle}.
 \item Change the value of \emph{ResultsDir} to a new directory.
 \item Change the value of \emph{ResultsPickle} to a new file.
 \item Change \emph{Time} in search to the maximal runtime (in seconds), e.g. $300$.
 \item Set \emph{FullTime} to True.
 \item Set \emph{TryWithNewDefaultTime} to True.
\end{enumerate}
  \item (Optional) Run \emph{findStrategies} again.
\begin{verbatim}
MaLeS/findStrategies.py
\end{verbatim}
  \item The newly found strategies are stored in \emph{ResultsDir}. MaLeS can now learn from these strategies via
\begin{verbatim}MaLeS/learn.py\end{verbatim} 
\end{enumerate}

For completeness, Table~\ref{tab:setup.ini} contains a list of all parameters in \emph{setup.ini} with their descriptions.

\begin{table}[htbp]
\centering
\caption{Parameters of MaLeS}
\begin{tabular}{lp{0.65\linewidth}}
\toprule
Settings Parameter & Description \\\midrule
\texttt{TPTP}   & The TPTP directory. Not required.\\
\texttt{TmpDir}    & Directory for temporary files.\\
\texttt{Cores} & How many cores to use.\\
\texttt{ResultsDir} & Directory where the results of the findStrategies are stored.\\
\texttt{ResultsPickle} & Directory where the models are stored.\\
\texttt{TmpResultsDir} & Like \emph{ResultsDir}, but only used if \emph{TryWithNewDefaultTime} is True.\\
\texttt{TmpResultsPickle} & Like \emph{ResultsPickle}, but only used if \emph{TryWithNewDefaultTime} is True.\\
\texttt{Clear} & If True, all existing results are ignored and MaLeS starts from scratch.\\
\texttt{LogToFile} & If True, a log file is created.\\
\texttt{LogFile} & Name of the log file.\\
& \\
Search Parameter & Description \\\midrule
\texttt{Time}   & Maximal runtime during search.\\
\texttt{Problems}    & File with the absolute pathnames of the problems.\\
\texttt{FullTime} & If True, the ATP is run for the value of \emph{Time}. If False, it is run for the rounded minimal time required to solve the problem.\\
\texttt{TryWithNewDefaultTime} & If True, findStrategies uses the best strategies from \emph{TmpResultsDir} and \emph{TmpResultsPickle} as a start strategies for a new search.\\
\texttt{Walks} & How many different strategies are tried in the local search step.\\
\texttt{WalkLength} & Up to this many parameters are changed for each strategy in the local search step.\\
& \\
Learn Parameter & Description \\\midrule
\texttt{Features} & Which features to use. Possible values are \emph{E} for the E features and \emph{TPTP} for the TPTP features.\\
\texttt{FeaturesFile} & Location of the feature file.\\
\texttt{StrategiesFile} & Location of the strategies file.\\
\texttt{KernelFile} & Location of the file containing the kernel matrices.\\
\texttt{RegularizationGrid} & Possible values for $\lambda$.\\
\texttt{KernelGrid} & Possible values for $\sigma$.\\
\texttt{CrossValidate} & If False, no crossvalidation is done during learning. Instead the first values in \emph{RegularizationGrid} and \emph{KernelGrid} are used.\\
\texttt{CrossValidationFolds} & How many folds to use during crossvalidation.\\
\texttt{StartStrategies} & Number of start strategies.\\
\texttt{StartStrategiesTime} & Runtime of each start strategy.\\
\texttt{CPU Bias} & This value is added to each runtime before learning. Serves as a buffer against runtime irregularities.\\
\texttt{Tolerance} & For a strategy $s$ to be considered as a good strategy, there must be at least one problem where the difference of the best runtime of any strategy and the runtime of $s$ is at most this value.\\
& \\
Run Parameter & Description \\\midrule
\texttt{CPUSpeedRatio}   & Predicted runtimes are multiplied with this value. Useful if the training was done on a different machine.\\
\texttt{MinRunTime}    & Minimal time a strategy is run.\\
\texttt{Features} & Either \emph{TPTP} for higher order features or \emph{E} for first order features.\\
\texttt{StrategiesFile} & Location of the strategies file.\\
\texttt{FeaturesFile} & Location of the feature file.\\
\texttt{OutputFile} & If not None, the output of MaLeS is stored in this file.\\
\bottomrule
\end{tabular}
\label{tab:setup.ini}
\end{table}

\subsubsection{Using a New Prover}
\label{sec:New_Prover}
The behavior of MaLeS is defined in three configuration files: \emph{ATP.ini} defines the ATP and its parameters, 
\emph{setup.ini} configures the searching and learning of MaLeS and \emph{strategies.ini} contains the default strategies of the ATP 
that form the starting point of the strategy search for the set of preselected strategies.
To use a new prover, \emph{ATP.ini} and \emph{strategies.ini} need to be adapted.
Table~\ref{tab:ATP.ini} describes the parameters in \emph{ATP.ini}.

\begin{table}[htbp!]
\centering
\caption{Parameters in ATP.ini}
\begin{tabular}{lp{0.6\linewidth}}
\toprule
ATP Settings Parameter & Description \\\midrule
\texttt{binary}   & Path to the ATP binary.\\
\texttt{time}    & Argument used to denote the time limit.\\
\texttt{problem} & Argument used to denote the problem.\\
\texttt{strategy} & Defines how parameters are given to the ATP. Three styles are supported: \emph{E}, \emph{LEO} and \emph{Satallax}.\\
\texttt{default} & Any default parameters that should always be used.\\
\bottomrule
\end{tabular}
\label{tab:ATP.ini}
\end{table}

The section \emph{Boolean Parameters} contains all flags that are given without a value.
\emph{List Parameters} contains flags which require a value and their possible values.
MaLeS searches strategies in the parameter space defined by \emph{Boolean Parameters} and \emph{List Parameters}.
Running \emph{EConfig.py} creates the configuration file for E which can serve an example.

Different ATPs have (unfortunately) different input formats for search parameters. 
MaLeS currently supports three formats: \emph{E}, \emph{LEO} or \emph{Satallax}.
Each format corresponds to the format of the respective ATP. 
Table~\ref{tab:formats} lists the differences. 
New formats need to be hardcoded in the file \emph{Strategy.py}.

\verbdef{\Etext}{---ordering=3 -sine13}
\verbdef{\LEOtext}{---ordering 3}
\verbdef{\Stext}{-m M}
\begin{table}[htbp!]
\centering
\caption{ATP Formats}
\begin{tabular}{lp{0.75\linewidth}}
\toprule
Format & Description \\\midrule
\texttt{E}   & Parameters and their values are joined by = if the parameter starts with -\--. Else the parameter is directly joined with its value. For example 
\Etext.\\
\texttt{LEO}    & Parameters and their values are joined by a space. For example
\LEOtext.\\
\texttt{Satallax} & The parameters are written in a new mode file $M$. The ATP is then called with ATP \Stext.\\
\bottomrule
\end{tabular}
\label{tab:formats}
\end{table}

Strategies defined in \emph{strategies.ini}  are used to initialize the strategy queue during the strategy searching for the set of preselected strategies.
The default ini format is used. Each strategy is its own section with each parameter on a separate line. For example \newpage
{\small{
\begin{verbatim}
[NewStrategy12884]
FILTER_START = 0
ENUM_IMP = 100
INITIAL_SUBTERMS_AS_INSTANTIATIONS = true
E_TIMEOUT = 1
POST_CONFRONT3_DELAY = 1000
FORALL_DELAY = 0
LEIBEQ_TO_PRIMEQ = true
\end{verbatim}
}
}
At least one strategy must be defined. 
After the ini files are adapted, the new ATP can be tuned and run using the procedure defined in the last two sections.

\subsection{CASC Results}
\label{sec:CASC}
MaLeS 1.2 is the third iteration of the MaLeS framework.
E-MaLeS 1.0 competed at CASC-23, E-MaLeS 1.1 at CASC@Turing and CASC-J6, 
and E-MaLeS 1.2 at CASC-24.
Satallax-MaLeS competed for the first time at CASC-24.
We give an overview of the older versions, the CASC performance and the changes over the years.

\subsubsection{CASC-23}
\label{sec:CASC-23}
E-MaLeS 1.0 \cite{Kuhlwein2012} was the first MaLeS version to compete at CASC. 
Stephan Schulz provided us with a set of strategies and information about their performance on all TPTP problems.
This data was used to train a kernel-based classification model for each strategy.
Given the features of a problem $p$, the classification models predict whether or not a strategy can solve $p$.
Altogether, three strategies were run. First E's auto mode for $60$ seconds, 
then the strategy with the highest probability of solving the problem as predicted by a Gaussian kernel classifier for $120$ seconds.
Finally the strategy with the highest probability of solving the problem as predicted by a linear (dot-product) kernel classifier was run for the remainder of the available time.
E-MaLeS 1.0 won third place in the FOF division. Table~\ref{tab:Sutcliffe2012} shows the results.

\begin{table}
\caption{Results of the FOF division of CASC 23}
\label{tab:Sutcliffe2012}
\centering
\begin{tabular}{lrrrr}
ATP &	Vampire 0.6 & Vampire 1.8  & E-MaLeS 1.0 & EP 1.4 pre \\
\midrule
Solved & $269 / 300$ & $263 / 300$ & $233 / 300$ & $232 / 300$ \\
Average CPU Time & $12.95$ & $13.62$ & $18.85$& $22.55$ \\
\bottomrule
\end{tabular}
\end{table}

\subsubsection{CASC@Turing and CASC-J6}
\label{sec:CASC@Turing}
E-MaLeS 1.1 \cite{Kuehlwein2013} changed the learning from classification to regression. 
Like E-MaLeS 1.0, E-MaLeS 1.1 learned from (an updated version of) Schulz's data.
Instead of predicting which strategy to run, E-MaLeS 1.1 learned runtime prediction functions.
The learning method is the same as the one presented in this chapter, without the updating of the prediction functions.
E-MaLeS 1.1 first ran E's auto mode for $60$ seconds. 
Afterwards, each strategy was run for its predicted runtime, starting with the strategy with the lowest predicted runtime.
E-MaLeS 1.1 won second place in the FOF divisions of both CASC@Turing (Table~\ref{tab:Sutcliffe2013}) and CASC-J6 (Table~\ref{tab:CASCTuring}).
It also came fourth in the LTB division of CASC-J6.

\begin{table}
\caption{Results of the FOF division of CASC-J6}
\label{tab:Sutcliffe2013}
\centering
\begin{tabular}{lrrrr}
ATP &	Vampire 2.6 & E-MaLeS 1.1 & EP 1.6 pre & Vampire 0.6 \\
\midrule
Solved & $429 / 450$ & $377 / 450$ & $359 / 450$ & $355 / 450$ \\
Average CPU Time & $13.17$ & $17.85$ & $13.46$& $11.81$ \\
\bottomrule
\end{tabular}
\end{table}

\begin{table}
\caption{Results of the FOF division of CASC@Turing}
\label{tab:CASCTuring}
\centering
\begin{tabular}{lrrrr}
ATP &	Vampire 2.6 & E-MaLeS 1.1 & EP 1.6 pre & Vampire 0.6 \\
\midrule
Solved & $469 / 500$ & $401 / 500$ & $378 / 500$ & $368 / 500$ \\
Average CPU Time & $20.26$ & $20.81$ & $14.49$& $16.40$ \\
\bottomrule
\end{tabular}
\end{table}

\subsubsection{CASC-24}
\label{sec:CASC-24}
E-MaLeS 1.2 and Satallax-MaLeS 1.2 competed at CASC 24, both based on the algorithms presented in this chapter.
E-MaLeS 1.2 used Schulz's strategies as start strategies for \textit{find\_strategies}.
It is the first E-MaLeS that was not based on the CASC version of E (E 1.7 in E-MaLeS 1.2 vs E 1.8). 
E-MaLeS 1.2 got fourth place in the FOF division, losing to two versions of Vampire, and E 1.8.
Several significant changes were introduced in E 1.8, in particular new strategies and E's own strategy scheduling.
Satallax-MaLeS won first place in the THF division before Satallax.
The results can be seen in Tables~\ref{tab:CASC24FOF} and \ref{tab:CASC24THF}.

\begin{table}
\caption{Results of the FOF division of CASC 24}
\label{tab:CASC24FOF}
\centering
\begin{tabular}{lrrrr}
ATP &	Vampire 2.6 & Vampire 3.0 & EP 1.8 & E-MaLeS 1.2  \\
\midrule
Solved & $281 / 300$ & $274 / 300$ & $249 / 300$ & $237 / 300$ \\
Average CPU Time & $12.24$ & $10.91$ & $29.02$& $14.52$ \\
\bottomrule
\end{tabular}
\end{table}

\begin{table}
\caption{Results of the THF division of CASC 24}
\label{tab:CASC24THF}
\centering
\begin{tabular}{lrrr}
ATP &	Satallax-MaLeS 1.2 & Satallax & Isabelle 2013 \\
\midrule
Solved & $119 / 150$ & $116 / 150$ & $108 / 150$ \\
Average CPU Time & $10.42$ & $11.39$ & $54.65$ \\
\bottomrule
\end{tabular}
\end{table}

\end{document}